\documentclass[accepted,colorlinks,citecolor=blue,urlcolor=blue,linkcolor=blue]{uai2023} 
\usepackage[american]{babel}

\usepackage{natbib} 
\usepackage{mathtools} 
\usepackage{booktabs} 
\usepackage{tikz} 
 
\usepackage{comment}
\usepackage{algorithmic,algorithm}

\usepackage{microtype}
\usepackage{graphicx}

\usepackage{multirow,makecell}
\usepackage{amsmath,amssymb,amsthm,bm}
\usepackage{mathtools}
\usepackage{bbm}
\usepackage{enumerate}
\usepackage{subfigure}
\usepackage{color}

\usepackage[utf8]{inputenc} 
\usepackage[T1]{fontenc}    
\usepackage{url}            
\usepackage{booktabs}       
\usepackage{amsfonts}       
\usepackage{nicefrac}       
\usepackage{microtype}      
\usepackage{xcolor}         

\usepackage{nameref}
\usepackage{zref-xr,zref-user}
\zxrsetup{toltxlabel}
\zexternaldocument*{camera_ready_supp}

\newtheorem{theorem}{Theorem}
\newtheorem{lemma}[theorem]{Lemma}
\newtheorem{definition}[theorem]{Definition}
\newtheorem{remark}[theorem]{Remark}

\newtheorem{assumption}[theorem]{Assumption}

\newtheorem{proposition}[theorem]{Proposition}

\newcommand{\argmin}{\mathop{\mathrm{argmin}}}
\newcommand{\argmax}{\mathop{\mathrm{argmax}}}

\makeatletter
\newcommand{\printfnsymbol}[1]{%
  \textsuperscript{\@fnsymbol{#1}}%
}
\makeatother

\def\R{\mathbb{R}}
\def\E{\mathbb{E}}

\def\1{\mathbbm{1}}

\def\half{\frac{1}{2}}

\def\cF{\mathcal{F}}

\def\cH{\mathcal{H}}

\def\cW{\mathcal{W}}
\def\cX{\mathcal{X}}
\def\cY{\mathcal{Y}}

\title{No-Regret Linear Bandits beyond Realizability}

\author[1]{Chong Liu}
\author[1]{Ming Yin}
\author[1]{Yu-Xiang Wang}
\affil[1]{
    Department of Computer Science\\
    University of California\\
    Santa Barbara, CA 93106, USA
}

\begin{document}

\maketitle

\begin{abstract}
We study linear bandits when the underlying reward function is \emph{not} linear. Existing work relies on a uniform misspecification parameter $\epsilon$ that measures the sup-norm error of the best linear approximation. This results in an unavoidable linear regret whenever $\epsilon > 0$. We describe a more natural model of misspecification which only requires the approximation error at each input $x$ to be proportional to the suboptimality gap at $x$.  It captures the intuition that, for optimization problems, near-optimal regions should matter more and we can tolerate larger approximation errors in suboptimal regions. Quite surprisingly, we show that the classical LinUCB algorithm --- designed for the realizable case --- is automatically robust against such gap-adjusted misspecification.  It achieves a near-optimal $\sqrt{T}$ regret for problems that the best-known regret is almost linear in time horizon $T$. Technically, our proof relies on a novel self-bounding argument that bounds the part of the regret due to misspecification by the regret itself. 

\end{abstract}

\section{Introduction}\label{sec:intro}
Stochastic linear bandit is a classical problem of online learning and decision-making with many influential applications, e.g., A/B testing \citep{claeys2021dynamic}, recommendation systems \citep{chu2011contextual}, advertisement placements \citep{wang2021hybrid}, clinical trials \citep{moradipari2020stage}, hyperparameter tuning \citep{alieva2021robust}, and new material discovery \citep{katz2020empirical}.

More formally, stochastic bandit is a sequential game between an agent who chooses a sequence of actions $x_0,...,x_{T-1}\in\cX$ and nature who decides on a sequence of noisy observations (rewards) $y_0,...,y_{T-1}$ according to $y_t = f_0(x_t) + \textit{noise}$ for some underlying function $f_0$.  The goal of the learner is to minimize the \emph{cumulative regret} the agent experiences relative to an oracle who knows the best action to choose ahead of time, i.e.,
$$
R_T(x_0,...,x_{T-1}) = \sum_{t=0}^{T-1} r_t = \sum_{t=0}^{T-1} \max_{x\in \cX} f_0(x) - f_0(x_t),
$$
where $r_t$ is called \emph{instantaneous regret}.

Despite being highly successful in the wild, existing theory for stochastic linear bandits (or more generally learning-oracle based bandits problems \citep{foster2018practical,foster2020beyond}) relies on a \emph{realizability} assumption, i.e., the learner is given access to a function class $\cF$ such that the true expected reward $f_0: \cX\rightarrow \R$ satisfies that $f_0\in\cF$. Realizability is considered one of the strongest and most restrictive assumptions in the standard statistical learning setting, but in the linear bandits, 
all known attempts to deviate from the realizability assumption result in a regret that grows linearly with $T$ \citep{ghosh2017misspecified,lattimore2020learning,zanette2020learning,neu2020efficient,bogunovic2021misspecified,krishnamurthy2021tractable}. 
 
In practical applications, it is often observed that feature-based representation of the actions with function approximations in
estimating the reward can result in very strong policies even if the estimated reward functions are far from
being correct \citep{foster2018practical}. 

So what went wrong? The critical intuition we rely on is the following:
\begin{quote}
It should be sufficient for the estimated reward function to clearly \emph{differentiate} good actions from bad ones, rather than requiring it to perfectly estimate the rewards numerically.
\end{quote}

\textbf{Contributions.} 
In this paper, we formalize this intuition by defining a new family of misspecified bandit problems based on a condition that adjusts the need for an accurate approximation pointwise at every $x\in\cX$ according to the suboptimality gap at $x$. Unlike the existing misspecified linear bandits problems with a linear regret, our problem admits a nearly optimal $\tilde{O}(\sqrt{T})$ regret despite being heavily misspecified. Specifically: 
\begin{itemize}
 \item We define $\rho$-\emph{gap-adjusted misspecified} ($\rho$-GAM) function approximations and characterize how they preserve important properties of the true function that are relevant for optimization.
 \item We show that the classical LinUCB algorithm \citep{abbasi2011improved} can be used \emph{as is} (up to some mild hyperparameters) to achieve an $\tilde{O}(\sqrt{T})$ regret under a moderate level of gap-adjusted misspecification ($\rho \leq O(1/\sqrt{\log T})$). In comparison, the regret bound one can obtain under the corresponding uniform misspecification setting is only $\tilde{O}(T/\sqrt{\log T})$. This represents an exponential improvement in the average regret metric $R_T/T$.
\end{itemize}

To the best of our knowledge, the suboptimality gap-adjusted misspecification problem was not studied before and we are the first to obtain $\sqrt{T}$-style regrets without a realizability assumption.

\textbf{Technical novelty.} Due to misspecification, we have technical challenges that appear in bounding the instantaneous regret and parameter uncertainty region. We tackle the challenges by a self-bounding trick, i.e., bounding the instantaneous regret by the instantaneous regret itself, which can be of independent interest in more settings, e.g., Gaussian process bandit optimization and reinforcement learning.

\section{Related Work}\label{sec:rw}

The problem of linear bandits was first introduced in \citet{abe1999associative}. Then \citet{auer2002finite} proposed the upper confidence bound to study linear bandits where the number of actions is finite. Based on it, \citet{dani2008stochastic} proposed an algorithm based on confidence ellipsoids and then \citet{abbasi2011improved} simplified the proof with a novel self-normalized martingale bound. Later \citet{chu2011contextual} proposed a simpler and more robust linear bandit algorithm and showed $\tilde{O}(\sqrt{d T})$ regret cannot be improved beyond a polylog factor. \citet{li2019nearly} further improved the regret upper and lower bound, which characterized the minimax regret up to an iterated logarithmic factor. See \citet{lattimore2020bandit} for a detailed survey of linear bandits. 

In terms of misspecification, \citet{ghosh2017misspecified} first studied the misspecified linear bandit with a fixed action set. They found that LinUCB \citep{abbasi2011improved} is not robust when misspecification is large. They showed that in a favourable case when one can test the linearity of the reward function, their RLB algorithm is able to switch between the linear bandit algorithm and finite-armed bandit algorithm to address misspecification issue and achieve the $\tilde{O}(\min \{\sqrt{K},d\}\sqrt{T})$ regret where $K$ is number of arms.

The most studied setting of model misspecification is uniform misspecification where the $\ell_\infty$ distance between the best-in-class function and the true function is always upper bounded by some parameter $\epsilon$, i.e.,
\begin{definition}[$\epsilon$-uniform misspecification]
    We say function class $\cF$ is an $\epsilon$-uniform misspecified approximation of $f_0$ if there exists $f\in \cF$ such that $\sup_{x\in\cX}|f(x) - f_0(x)| \leq \epsilon$.
\end{definition}
Under this definition, \citet{lattimore2020learning} proposed the optimal design-based phased elimination algorithm for misspecified linear bandits and achieved $\tilde{O}(d\sqrt{T} + \epsilon \sqrt{d} T)$ regret when number of actions is infinite. They also found that with modified confidence band in LinUCB, LinUCB is able to achieve the same regret. With the same misspecification model, \citet{foster2020beyond} studied contextual bandit with regression oracle, \citet{neu2020efficient} studied multi-armed linear contextual bandit, and \citet{zanette2020learning} studied misspecified contextual linear bandits after reduction of the algorithm. All of their results suffer from linear regrets. Later \citet{bogunovic2021misspecified} studied misspecified Gaussian process bandit optimization problem and achieved $\tilde{O}(d\sqrt{T} + \epsilon \sqrt{d} T )$ regret when linear kernel is used in Gaussian process. Moreover, their lower bound shows that $\tilde{\Omega}(\epsilon T)$ term is unavoidable in this setting. 

Besides uniform misspecification, there are some work considering different definitions of misspecification. \citet{krishnamurthy2021tractable} defines misspecification error as an expected squared error between true function and best-in-class function where expectation is taken over distribution of context space and action space. \citet{foster2020adapting} considered average misspecification, which is weaker than uniform misspecification and allows tighter regret bound. However, they also have linear regrets. 
Our work is different from all related work mentioned above because we are working under a newly defined misspecifiation condition and show that LinUCB is a no-regret algorithm in this case.

Model misspecification is naturally addressed in the related \emph{agnostic} contextual bandits setting \citep{agarwal2014taming}, but these approaches typically require the action space to be finite, thus not directly applicable to our problem. In addition, empirical evidence \citep{foster2018practical} suggests that the regression oracle approach works better in practice than the agnostic approach even if realizability cannot be verified.

\begin{figure*}[t]
	\centering    
	\subfigure[$\rho$-gap-adjusted misspecification]{\label{fig:example}\includegraphics[width=0.45\linewidth]{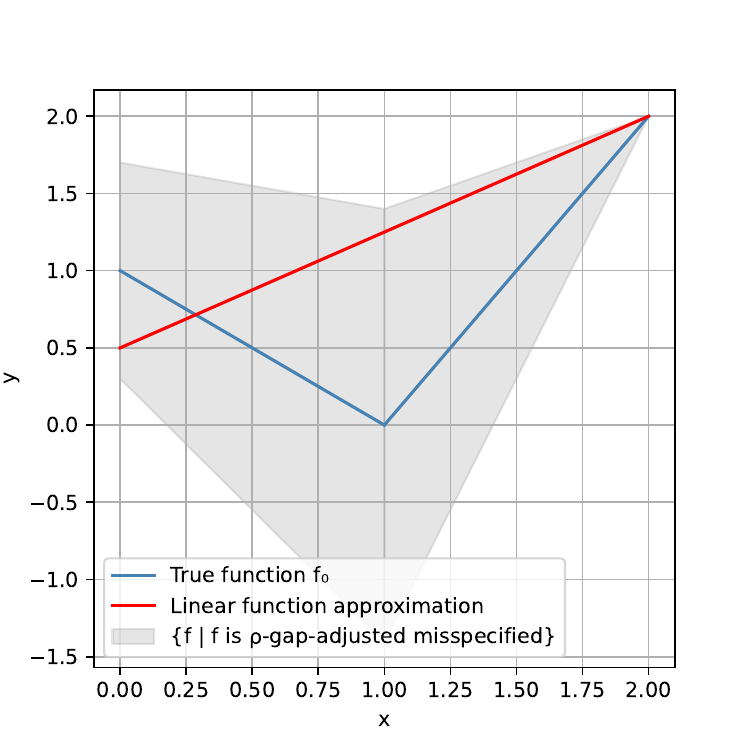}}
	\subfigure[Weak $\rho$-gap-adjusted misspecification]{\label{fig:example2}\includegraphics[width=0.45\linewidth]{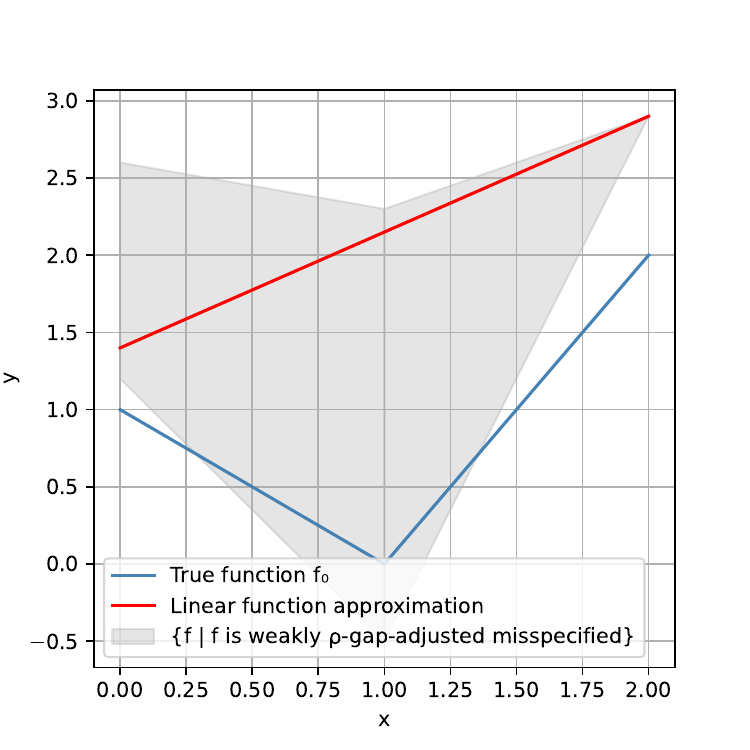}}
	\caption{(a): An example of $\rho$-gap-adjusted misspecification (Definition \ref{def:lm}) in $1$-dimension where $\rho=0.7$. The blue line shows a non-linear true function and the gray region shows the gap-adjusted misspecified function class. Note the vertical range of gray region at a certain point $x$ depends on the suboptimal gap. For example, at $x=1$ suboptimal gap is $2$ and the vertical range is $4\rho=2.8$. The red line shows a feasible linear function that is able to optimize the true function by taking $x_*=2$.  (b): An example of weak $\rho$-gap-adjusted misspecification (Definition \ref{def:lm_weak}) in $1$-dimension where $\rho=0.7$. The difference to Figure \ref{fig:example} is that one can shift the qualifying approximation arbitrarily up or down and the specified model only has to $\rho$-RAM approximate $f_0$ up to an additive constant factor.}
	\label{fig:main}
\end{figure*}

\section{Preliminaries}\label{sec:pre}

\subsection{Notations}\label{sec:notation}
Let $[n]$ denote the integer set $\{1,2,...,n\}$. The algorithm runs in $T$ rounds in total. Let $f_0$ denote the true function, so the maximum function value is defined as $f^* = \max_{x \in \cX} f_0(x)$ and the maximum point is defined as $x^* = \argmax_{x \in \cX} f_0(x)$. Let $\cX \subset \R^d$ and $\cY \subset \R$ denote the domain and range of $f_0$. We use $\cW$ to denote the parameter class of a family of linear functions $\cF := \{f_w: \cX \rightarrow \cY|w \in \cW\}$ where $f_w(x)=w^\top x$. Define $w_*$ as the parameter of best linear approximation function. $\|w\|_2 \leq C_w, \forall w \in \cW$ and $\|x\|_2 \leq C_b, \forall x \in \cX$.
For a vector $x$, its $\ell_2$ norm is denoted by $\|x\|_2 = \sqrt{\sum_{i=1}^d x^2_i}$ and for a matrix $A$ its operator norm is denoted by $\|A\|_\mathrm{op}$. For a vector $x$ and a square matrix $A$, define $\|x\|^2_A = x^\top A x$.

\subsection{Problem Setup}\label{sec:setup}
We consider the following optimization problem:
\begin{align*}
x_* = \argmax_{x \in \cX} f_0(x),
\end{align*}
where $f_0$ is the true function which might not be linear in $\cX$. We want to use a linear function $f_w=w^\top x\in\mathcal{F}$ to approximate $f_0$ and maximize $f_0$. At time $0\leq t \leq T-1$, after querying a data point $x_t$, we will receive a noisy feedback:
\begin{align}
y_t = f_0(x_t) + \eta_t, \label{eq:obs}
\end{align}
where $\eta_t$ is independent, zero-mean, and $\sigma$-sub-Gaussian noise.

The major highlight of our study is that we do not rely on the popular \emph{realizability} assumption (\emph{i.e.} $f_0\in\mathcal{F}$) that is frequently assumed in the existing function approximation literature. Alternatively, we propose the following gap-adjusted misspecification condition.

\begin{definition}[$\rho$-gap-adjusted misspecification]\label{def:lm}
We say a function $f$ is a $\rho$-gap-adjusted misspecified (or $\rho$-GAM in short) approximation of $f_0$ if 
for parameter $0 \leq \rho < 1$,
\begin{align*}
\sup_{x \in \cX} \left| \frac{f(x) - f_0(x)}{f^* - f_0(x)}\right|\leq \rho.\label{eq:local}
\end{align*}
We say function class $\cF=\{f_w | w\in\cW\}$ satisfies $\rho$-GAM for $f_0$, if 
there exists $w^*\in\cW$ such that $f_{w_*}$ is a $\rho$-GAM approximation of $f_0$.
\end{definition}
Observe that when $\rho = 0$, this recovers the standard realizability assumption, but when $\rho>0$ it could cover many misspecified function classes.

Figure~\ref{fig:example} shows a 1-dimensional example with $f_w(x)= 0.75x+0.5$ and piece-wise linear function $f_0(x)$ that satisfies local misspecification. With Definition~\ref{def:lm}, we have the following proposition. 

\begin{proposition}\label{prop:perservation}
Let $f$ be 
a $\rho$-GAM approximation of $f_0$ (Definition~\ref{def:lm}). Then it holds:
\begin{itemize}
\item (Preservation of maximizers) $$\argmax_{x}f(x) =\argmax_{x}f_{0}(x).$$
\item  (Preservation of max value) $$\max_{x\in\mathcal{X}}f(x)=f^*.$$
\item (Self-bounding property) $$|f(x) - f_0(x)| \leq \rho (f^* - f_0(x)) = \rho r(x).$$
\end{itemize}

\end{proposition}
This tells $f$ and $f_0$ coincide on the same global maximum points and the same global maxima if Definition \ref{def:lm} is satisfied, while allowing $f$ and $f_0$ to be different (potentially large) at other locations. Therefore, Definition~\ref{def:lm} is a ``local'' assumption that does not require $f$ to be uniformly close to $f_0$ (e.g. the ``uniform'' misspecification assumes $\sup_{x\in\mathcal{X}}|f(x)-f_0(x)|\leq \rho$). Proof of Proposition \ref{prop:perservation} is shown in Appendix \ref{sec:pres}.

In addition, we can modify Definition~\ref{def:lm} with a slightly weaker condition that only requires $\argmax_{x}f(x) =\argmax_{x}f_{0}(x)$
but not necessarily $\max_{x\in\mathcal{X}}f(x)=f^*$.

\begin{definition}[Weak $\rho$-gap-adjusted misspecification]\label{def:lm_weak}
Denote $f_w^*=\max_{x\in\mathcal{X}} f(x)$. Then we say $f$ is (weak) $\rho$-gap-adjusted misspecification approximation of $f_0$ for a parameter $0 \leq \rho < 1$ if:
\begin{align*}
\sup_{x \in \cX} \left| \frac{f(x) - f_w^*+f^*-f_0(x)}{f^* - f_0(x)}\right|\leq \rho.
\end{align*}
\end{definition}
See Figure \ref{fig:example2} for an example satisfying Definition \ref{def:lm_weak}, in which there is a constant gap between $f^*_w$ and $f^*$.
The idea of this weaker assumption is that we can always extend the function class by adding a single offset parameter $c$ w.l.o.g. to learn the constant gap $f^* - f_w^*$. In the linear case, this amounts to homogenizing the feature vector by appending $1$. For this reason, we stick to Definition~\ref{def:lm} and linear function approximation for conciseness and clarity in the main paper. See Appendix \ref{sec:weak} for formal statements and Appendix \ref{sec:weak_regret} for proofs of regret bound of linear bandits under Definition~\ref{def:lm_weak}.

Note that both Definition~\ref{def:lm} and Definition~\ref{def:lm_weak} are defined generically which do not require any assumptions on the parametric form of $f$. While we focus on the linear bandit setting in this paper, this notion can be considered for arbitrary function approximation learning problems.

\subsection{Assumptions}\label{sec:ass}
\begin{assumption}[Boundedness]\label{ass:boundedness}
For any $x\in\cX$, $\|x\|_2\leq C_b$.  For any $w\in\cW$, $\|w\|_2\leq C_w$. Moreover, for any $x,\tilde{x}\in\cX$, the true expected reward function $|f_0(x) - f_0(\tilde{x})| \leq F$.
\end{assumption}
These are mild assumptions that we assume for convenience. Relaxations of these are possible but not the focus of this paper. Note that the additional assumption is not required when $f_0$ is realizable.

\begin{assumption}\label{ass:unique}
Suppose $\mathcal{X}\in\R^d$ is a compact set, and all the global maximizers of $f_0$ live on the $d-1$ dimensional hyperplane. i.e., $\exists a\in\R^d,b\in \R^1$, s.t. 
\begin{align*}
\argmax_{x\in\mathcal{X}}f_{0}(x)\subset \{x\in\R^d: x^\top a=b\}.
\end{align*}
\end{assumption}
For instance, when $d=1$, the above reduces to that $f_0$ has a unique maximizer. This is a compatibility assumption for Definition~\ref{def:lm}, since any linear function that violates Assumption~\ref{ass:unique} will not satisfy Definition \ref{def:lm}.

In addition, to obtain an $\tilde{O}(\sqrt{T})$ regret, for any finite sample $T$, we require the following condition.
\begin{assumption}[Low misspecification]\label{ass:rho}
The linear function class is a $\rho$-GAM approximation of $f_0$ with
\begin{align}
\rho < \frac{1}{8 d \sqrt{\log \left(1 + \frac{T C^2_b C^2_w}{d \sigma^2}\right)}} = O\left( \frac{1}{d\sqrt{\log T}}\right).
\end{align}
\end{assumption}
The condition is required for technical reasons. Relaxing this condition for LinUCB may require fundamental breakthroughs that knock out logarithmic factors from its regret analysis. This will be further clarified in the proof.  In general, however, we conjecture that this condition is not needed and there are algorithms that can achieve $\tilde{O}(\sqrt{T}/(1-\rho))$ regret for any $\rho < 1$, but a new algorithm needs to be designed.

While this assumption may suggest that we still require realizability in a truly asymptotic world, handling a $O(1/\sqrt{\log T})$ level of misspecification is highly non-trivial in finite sample setting. For instance, if $T$ is a trillion, $1/\sqrt{\log (1e12)} \approx 0.19$. This means that for most practical cases, LinUCB is able to tolerate a constant level of misspecification under the GAM model.

\subsection{LinUCB Algorithm}\label{sec:alg}
We will focus on analyzing the classical Linear Upper Confidence Bound (LinUCB) algorithm due to \citep{dani2008stochastic,abbasi2011improved}, shown below.
\begin{algorithm}[!htbp]
\caption{LinUCB \citep{abbasi2011improved}}
	\label{alg:linucb}
	{\bf Input:}
	Predefined sequence $\beta_t$ for $t=1,2,3,...$ as in eq. \eqref{eq:beta_t};
 Set $\lambda=\sigma^2/C^2_w$ and $\mathrm{Ball}_0 = \cW$.
	\begin{algorithmic}[1]
	    \FOR{$t = 0,1,2,... $}
	    \STATE Select $x_t=\argmax_{x \in \cX} \max_{w \in \mathrm{Ball}_t} w^\top x$.
	    \STATE Observe $y_t = f_0(x_t) + \eta_t$.
     \STATE Update 
     \begin{align}
\Sigma_{t+1} = \lambda I + \sum_{i=0}^{t} x_i x^\top_i \mathrm{where}\  \Sigma_0 = \lambda I.\label{eq:sigma_t}
\end{align}
\STATE Update 
	    \begin{align}
\hat{w}_{t+1} = \argmin_w \lambda \|w\|^2_2+ \sum_{i=0}^{t} (w^\top x_i - y_i)^2_2.\label{eq:w_t_opt}
\end{align}
    \STATE Update $\mathrm{Ball}_{t+1} = \{w | \|w - \hat{w}_{t+1}\|^2_{\Sigma_{t+1}} \leq \beta_{t+1}\}.$
		\ENDFOR
	\end{algorithmic}
\end{algorithm}

\section{Main Results}\label{sec:theory}

In this section, we show that the classical LinUCB algorithm \citep{abbasi2011improved} works in $\rho$-gap-adjusted misspecified linear bandits and achieves cumulative regret at the order of $\tilde{O}(\sqrt{T}/(1-\rho))$. The following theorem shows the cumulative regret bound.

\begin{theorem}\label{thm:main}
Suppose Assumptions \ref{ass:boundedness}, \ref{ass:unique}, and \ref{ass:rho} hold. Set 
\begin{align}
\beta_t = 8\sigma^2 \left(1 + d\log\left(1+ \frac{t C^2_b C^2_w }{d \sigma^2} \right) + 2\log \left(\frac{\pi^2 t^2}{3\delta} \right)\right).\label{eq:beta_t}
\end{align} 
Then Algorithm \ref{alg:linucb} guarantees w.p. $> 1-\delta$ simultaneously for all $T=1,2,...$
\begin{align*}
R_T &\leq F + \sqrt{\frac{8 (T-1) \beta_{T-1} d}{(1-\rho)^2} \log \left( 1 + \frac{T C^2_b C^2_w }{d \sigma^2 }\right)}.
\end{align*}
\end{theorem}
\begin{remark}
The result shows that LinUCB achieves $\tilde{O}(\sqrt{T})$ cumulative regret bound and thus it is a no-regret algorithm in $\rho$-gap-adjusted misspecified linear bandits. In contrast, LinUCB can only achieve $\tilde{O}(\sqrt{T} + \epsilon T)$ regret in uniformly misspecified linear bandits. Even if $\epsilon = \tilde{O}(1/\sqrt{\log T})$, the resulting regret $\tilde{O}(T/\sqrt{\log T})$ is still exponentially worse than ours.
\end{remark}
\begin{proof}
By definition of cumulative regret, function range absolute bound $F$, and Cauchy-Schwarz inequality,
\begin{align*}
R_T &= r_0 + \sum_{t=1}^{T-1} r_t \\
&\leq F + \sqrt{\left(\sum_{t=1}^{T-1} 1 \right) \left(\sum_{t=1}^{T-1} r^2_t \right)}\\
&= F + \sqrt{ (T-1) \sum_{t=1}^{T-1} r^2_t}.
\end{align*}
Observe that the choice of $\beta_t$ is monotonically increasing in $t$. Also by Lemma~\ref{lem:w_t}, we get that with probability $1-\delta$, $w_*\in \text{Ball}_t, \forall t= 1,2,3,...$, which verifies the condition to apply Lemma \ref{lem:sos_r_t} simultaneously for all $T=1,2,3,...$, thereby completing the proof.
\end{proof}

\subsection{Regret Analysis}\label{sec:reg_ana}

The proof follows the LinUCB analysis closely. The main innovation is a self-bounding argument that controls the regret due to misspecification by the regret itself.  This appears in Lemma~\ref{lem:r_t} and then again in the proof of Lemma~\ref{lem:w_t}.

Before we proceed, let $\Delta_t$ denote the deviation term of our linear function from the true function at $x_t$, formally,
\begin{align}
\Delta_t = f_0(x_t) - w^\top_* x_t,\label{eq:delta}
\end{align}
And our observation model (eq. \eqref{eq:obs}) becomes
\begin{align}
y_t = f_0(x_t) + \eta_t = w_*^\top x_t + \Delta_t + \eta_t.\label{eq:obs2}
\end{align}
Moreover, we have the following lemma showing the property of deviation term $\Delta_t$.
\begin{lemma}[Bound of deviation term]\label{lem:delta}
$\forall t \in \{0,1,\ldots,T-1\}$,
\begin{align*}
|\Delta_t | \leq \frac{\rho}{1-\rho} w^\top_*(x_* - x_t).
\end{align*}
\begin{proof}
Recall the definition of deviation term in eq. \eqref{eq:delta}:
\begin{align*}
\Delta_t = f_0(x_t) - w_*^\top x_t.
\end{align*}
By Definition \ref{def:lm}, $\forall t \in \{0,1,\ldots,T-1\}$,
\begin{align*}
-\rho(f^* - f_0(x_t))\leq \Delta_t &\leq \rho(f^* - f_0(x_t))\\
-\rho(f^* - w_*^\top x_t - \Delta_t)\leq \Delta_t &\leq \rho(f^* - w_*^\top x_t - \Delta_t)\\
-\rho(w_*^\top x_* - w_*^\top x_t - \Delta_t)\leq \Delta_t &\leq \rho(w_*^\top x_* - w_*^\top x_t - \Delta_t)\\
\frac{-\rho}{1-\rho} (w_*^\top x_* - w_*^\top x_t)\leq \Delta_t &\leq \frac{\rho}{1 + \rho}(w_*^\top x_* - w_*^\top x_t),
\end{align*}
where the third line is by Proposition \ref{prop:perservation} and the proof completes by taking the absolute value of the lower and upper bounds.
\end{proof}
\end{lemma}

Next, we prove instantaneous regret bound and its sum of squared regret version in the following two lemmas:

\begin{lemma}[Instantaneous regret bound]\label{lem:r_t}
Define $u_t := \| x_t\|_{\Sigma_t^{-1}}$, assume $w_*\in \mathrm{Ball}_t$
then for each $t\geq 1$
\begin{align*}
r_t \leq \frac{2\sqrt{\beta_t}u_t}{1-\rho}.
\end{align*}
\end{lemma}
\begin{proof}
By definition of instantaneous regret,
\begin{align*}
r_t &= f^* - f_0(x_t)\\
&= w^\top_* x_* - (w^\top_* x_t + \Delta(x_t))\\
&\leq w^\top_* x_* - w^\top_* x_t + \rho (f^* - f_0(x_t))\\
&= w^\top_* x_* - w^\top_* x_t + \rho r_t,
\end{align*}
where the inequality is by Definition \ref{def:lm}. Therefore, by rearranging the inequality we have
\begin{align*}
r_t &\leq \frac{1}{1-\rho}(w^\top_* x_* - w^\top_* x_t) \leq  \frac{2\sqrt{\beta_t} u_t}{1-\rho},
\end{align*}
where the last inequality is by Lemma \ref{lem:gap}.
\end{proof}

\begin{lemma}\label{lem:sos_r_t}
Assume $\beta_t$ is monotonically nondecreasing and $w_*\in \mathrm{Ball}_t$ for all $t=1,...,T-1$, then 
\begin{align*}
    \sum_{t=1}^{T-1} r^2_t \leq \frac{8\beta_{T-1} d}{(1-\rho)^2} \log \left( 1 + \frac{T C^2_b}{d \lambda }\right).
\end{align*}
\end{lemma}
\begin{proof}
By definition $u_t = \sqrt{x^\top_t \Sigma^{-1}_{t} x_t}$ and Lemma \ref{lem:r_t},
\begin{align*}
\sum_{t=1}^{T-1} r^2_t &\leq \sum_{t=1}^{T-1} \frac{4}{(1-\rho)^2} \beta_t u^2_t \\
&\leq \frac{4\beta_{T-1}}{(1-\rho)^2} \sum_{t=1}^{T-1} u^2_t \\
&\leq \frac{4\beta_{T-1}}{(1-\rho)^2} \sum_{t=0}^{T-1} u^2_t\\
&\leq \frac{8\beta_{T-1} d}{(1-\rho)^2} \log \left( 1 + \frac{T C^2_b}{d \lambda }\right),
\end{align*}
where the second inequality is by the monotonic increasing property of $\beta_t$ and the last inequality uses the elliptical potential lemma (Lemma \ref{lem:sum_pos}).
\end{proof}

Previous two lemmas hold on the following lemma, bounding the gap between $f^*$ and the linear function value at $x_t$, shown below.

\begin{lemma}\label{lem:gap}
Define $u_t = \| x_t\|_{\Sigma_t^{-1}}$ and assume $\beta_t$ is chosen such that $w_*\in \mathrm{Ball}_t$.
Then
\begin{align*}
w_*^\top (x_* - x_t) \leq 2 \sqrt{\beta_t} u_t.
\end{align*}
\end{lemma}
\begin{proof}
Let $\tilde{w}$ denote the parameter that achieves $\argmax_{w \in \mathrm{Ball}_t} w^\top x_t$, by the optimality of $x_t$, 
\begin{align*}
\ w_*^\top x_* - w^\top_* x_t &\leq \tilde{w}^\top x_t - w^\top_* x_t \\
&= (\tilde{w} - \hat{w}_t + \hat{w}_t - w_*)^\top x_t\\
&\leq \|\tilde{w} - \hat{w}_t\|_{\Sigma_t} \|x_t\|_{\Sigma^{-1}_t} \\
&\quad \ + \|\hat{w}_t - w_*\|_{\Sigma_t} \|x_t\|_{\Sigma^{-1}_t}\\
&\leq 2\sqrt{\beta_t} u_t
\end{align*}
where the second inequality applies Holder's inequality; the last line uses the definition of $\mathrm{Ball}_t$ (note that both $w_*,\tilde{w}\in \mathrm{Ball}_t).$
\end{proof} 

\subsection{Confidence Analysis}\label{sec:conf_ana}
All analysis in the previous section requires $w_* \in \mathrm{Ball}_t, \forall t\in [T]$. In this section, we show that our choice of $\beta_t$ in \eqref{eq:beta_t} is valid and $w_*$ is trapped in the uncertainty set $\mathrm{Ball}_t$ with high probability.

\begin{lemma}[Feasibility of $\mathrm{Ball}_t$]\label{lem:w_t}
Suppose Assumptions \ref{ass:boundedness}, \ref{ass:unique}, and \ref{ass:rho} hold. Set $\beta_t$ as in eq. \eqref{eq:beta_t}. Then, w.p. $> 1- \delta$,
\begin{align*}
\|w_* - \hat{w}_t\|^2_{\Sigma_t} \leq \beta_t, \forall t=1,2,...
\end{align*}
\end{lemma}
\begin{proof}
By setting the gradient of objective function in eq. \eqref{eq:w_t_opt} to be $0$, we obtain the closed form solution of eq. \eqref{eq:w_t_opt}:
\begin{align*}
\hat{w}_t = \Sigma_t^{-1} \sum_{i=0}^{t-1} y_i x_i.
\end{align*}
Therefore,
\begin{align}
\hat{w}_t - w_* &= - w_* + \Sigma_t^{-1} \sum_{i=0}^{t-1} x_i y_i \nonumber\\
&= - w_* + \Sigma_t^{-1} \sum_{i=0}^{t-1} x_i (x_i^\top w_* + \eta_i + \Delta_i) \nonumber\\
&= -w_* + \Sigma^{-1}_t \left(\sum_{i=0}^{t-1} x_i x_i^\top \right) w_* + \Sigma^{-1}_t \sum_{i=0}^{t-1} \eta_i x_i  \nonumber \\
&\quad \ + \Sigma^{-1}_t \sum_{i=0}^{t-1} \Delta_i x_i,\label{eq:w_t_1}
\end{align}
where the second equation is by eq. \eqref{eq:obs2} and the first two terms of eq. \eqref{eq:w_t_1} can be further simplified as
\begin{align*}
&\quad \ -w_* + \Sigma^{-1}_t \left(\sum_{i=0}^{t-1} x_i x_i^\top \right) w_* \\
&= -w_* + \Sigma^{-1}_t \left(\lambda I + \sum_{i=0}^{t-1} x_i x_i^\top - \lambda I \right) w_*\\
&= - w_* + \Sigma_t^{-1} \Sigma_t w_* - \lambda \Sigma_t^{-1} w_*\\
& = - \lambda \Sigma^{-1}_t w_*,
\end{align*}
where the second equation is by definition of $\Sigma_t$ (eq. \eqref{eq:sigma_t}). Therefore, eq. \eqref{eq:w_t_1} can be rewritten as
\begin{align*}
\hat{w}_t - w_* = - \lambda \Sigma^{-1}_t w_*  + \Sigma^{-1}_t \sum_{i=0}^{t-1} \eta_i x_i  + \Sigma^{-1}_t \sum_{i=0}^{t-1} \Delta_i x_i.
\end{align*}
Multiply both sides by $\Sigma_t^{\half}$ and we have
\begin{align*}
\Sigma_t^{\half}(\hat{w}_t - w_*) &= - \lambda \Sigma^{-\half}_t w_* + \Sigma_t^{-\half} \sum_{i=0}^{t-1} \eta_i x_i \\
&\quad \ + \Sigma^{-\half}_t \sum_{i=0}^{t-1} \Delta_i x_i.
\end{align*}
Take a square of both sides and apply generalized triangle inequality, we have
\begin{align}
\|\hat{w}_t - w_*\|^2_{\Sigma_t} & \leq 4 \lambda^2 \|w_*\|^2_{\Sigma_t^{-1}} + 4\left\| \sum_{i=0}^{t-1} \eta_i x_i \right\|^2_{\Sigma_t^{-1}} \nonumber \\
&\quad \ + 4\left\| \sum_{i=0}^{t-1} \Delta_i x_i \right\|^2_{\Sigma_t^{-1}}.\label{eq:w_t_2}
\end{align}
The remaining task is to bound these three terms separately. The first term of eq. \eqref{eq:w_t_2} is bounded as
\begin{align*}
4\lambda^2 \|w_*\|^2_{\Sigma^{-1}_t} \leq 4 \lambda \|w_*\|^2_2 \leq 4\sigma^2,
\end{align*}
where the first inequality is by definition of $\Sigma_t$ and $\|\Sigma^{-1}_t\|_\mathrm{op} \leq 1/\lambda$ and the second inequality is by choice of $\lambda = \sigma^2/C^2_w$.

The second term of eq. \eqref{eq:w_t_2} can be bounded by Lemma \ref{lem:self_norm} and Lemma \ref{lem:potential}:
\begin{align*}
4 \left\|\sum_{i=0}^{t-1} \eta_i x_i \right\|^2_{\Sigma_t^{-1}} &\leq 4\sigma^2 \log \left(\frac{\det (\Sigma_t) \det(\Sigma_0)^{-1}}{\delta_t^2} \right)\\
&\leq 4\sigma^2 \left(d \log\left(1 + \frac{t C^2_b}{d \lambda} \right) - \log \delta^2_t \right),
\end{align*}
where $\delta_t$ is chosen as $3\delta/(\pi^2 t^2)$ so that the total failure probabilities over $T$ rounds can always be bounded by $\delta/2$:
\begin{align*}
\sum_{t=1}^T \frac{3\delta}{\pi^2 t^2} < \sum_{t=1}^\infty \frac{3\delta}{\pi^2 t^2} = \frac{3\delta \pi^2 }{6 \pi^2} = \frac{\delta}{2}.
\end{align*}

And the third term of eq. \eqref{eq:w_t_2} can be bounded as
\begin{align*}
4 \left \| \sum_{i=0}^{t-1} \Delta_i x_i \right\|^2_{\Sigma^{-1}_t} &= 4\left(\sum_{i=0}^{t-1} \Delta_i x_i \right)^\top \Sigma^{-1}_t \left(\sum_{j=0}^{t-1} \Delta_j x_j \right)\\
&= 4 \sum_{i=0}^{t-1} \sum_{j=0}^{t-1} \Delta_i \Delta_j x_i \Sigma^{-1}_t x_j\\
&\leq 4\sum_{i=0}^{t-1} \sum_{j=0}^{t-1} |\Delta_i| |\Delta_j| \|x_i\|_{\Sigma^{-1}_t} \|x_j\|_{\Sigma^{-1}_t},
\end{align*}
where the last line is by taking the absolute value and Cauchy-Schwarz inequality. Continue the proof and we have
\begin{align*}
&\quad \ 4\sum_{i=0}^{t-1} \sum_{j=0}^{t-1} |\Delta_i| |\Delta_j| \|x_i\|_{\Sigma^{-1}_t} \|x_j\|_{\Sigma^{-1}_t} \\
&= 4\left( \sum_{i=0}^{t-1} |\Delta_i|  \|x_i\|_{\Sigma^{-1}_t}\right) \left(\sum_{j=0}^{t-1} |\Delta_j| \|x_j\|_{\Sigma^{-1}_t}\right)\\
&= 4\left( \sum_{i=0}^{t-1} |\Delta_i|  \|x_i\|_{\Sigma^{-1}_t}\right)^2\\
&\leq 4 \left(\sum_{i=0}^{t-1} |\Delta_i|^2 \right) \left(\sum_{i=0}^{t-1} \|x_j\|_{\Sigma^{-1}_t}^2 \right)\\
&\leq 4 d \rho^2 \sum_{i=0}^{t-1} r_i^2 .
\end{align*}
where the first inequality is due to Cauchy-Schwarz inequality and the second uses  the self-bounding properties $|\Delta_i| \leq \rho r_i$ from Proposition~\ref{prop:perservation} and Lemma~\ref{lem:sum_pos2}.

To put things together, we have shown that w.p. $> 1-\delta$, for any $t\geq 1$,
\begin{align}
&\quad \ \|\hat{w}_t-w_*\|^2_{\Sigma_t^{-1}} \nonumber \\
&\leq 4 \sigma^2 + 4\rho^2 d \sum_{i=0}^{t-1} r_i^2 \nonumber \\
&\quad \ + 4\sigma^2 \left(d\log\left(1+ \frac{t C^2_b }{d \lambda} \right) + 2\log \left(\frac{\pi^2 t^2}{3\delta} \right)\right) , \label{eq:radius}
\end{align}
where we condition on \eqref{eq:radius} for the rest of the proof.

Observe that this implies that the feasibility of $w_*$ in $\mathrm{Ball}_t$ can be enforced if we choose $\beta_t$ to be larger than \eqref{eq:radius}. The feasiblity of $w_*$ in turn allows us to apply Lemma~\ref{lem:r_t} to bound the RHS with $\beta_{0},...,\beta_{t-1}$. We will use induction to prove that our choice 
$$\beta_t := 2\sigma^2\iota_t \text{ for } t=1,2,...$$ is valid, where short hand $$\iota_t:=4 + 4\left(d\log\left(1+ \frac{t C^2_b }{d \lambda} \right) + 2\log \left(\frac{\pi^2 t^2}{3\delta} \right)\right).$$

For the base case $t=1$, by eq. \eqref{eq:radius} and the definition of $\beta_1$ we directly have $\|\hat{w}_1-w_*\|^2_{\Sigma_1^{-1}}\leq \beta_1$. Assume our choice of $\beta_i$ is feasible for $i=1,...,t-1$, then we can write
\begin{align*}
    \|\hat{w}_t-w_*\|^2_{\Sigma_t^{-1}} &\leq \sigma^2\iota_t + 4\rho^2 d \sum_{i=1}^{t-1} \beta_i u_i^2 \\
    &\leq  \sigma^2\iota_t + 4\rho^2 d \beta_{t-1}\sum_{i=1}^{t-1} u_i^2,
    \end{align*}
where the second line is due to non-decreasing property of $\beta_t$. Then by Lemma \ref{lem:sum_pos} and Assumption~\ref{ass:rho}, we have
\begin{align}
\|\hat{w}_t-w_*\|^2_{\Sigma_t^{-1}}    &\leq \sigma^2\iota_t +8\rho^2 d^2 \beta_{t-1}\log \left(1+\frac{tC_b^2}{d\lambda} \right) \nonumber\\
    &\leq \sigma^2\iota_t + \half \beta_{t-1} \leq 2\sigma^2\iota_t = \beta_{t},
    \label{eq:known_rho}
\end{align}

The critical difference from the standard LinUCB analysis here is that if $\beta_{t-1}$ appears on the LHS of the bound and if its coefficient is larger, any valid bound for $\beta_t$ will have to grow exponentially in $t$. This is where Assumption \ref{ass:rho} helps us. Assumption \ref{ass:rho}  ensures that the coefficient of $\beta_{t-1}$ is smaller than $1/2$, so we can take $\beta_{t-1}\leq \beta_t$ and move $\beta_t/2$ to the right-hand side. 
\end{proof}

Proof of previous lemma needs the following two lemmas.

\begin{lemma}[Upper bound of $\sum_{i=0}^{t-1} x^\top_i \Sigma_t^{-1} x_i$]\label{lem:sum_pos2}
\begin{align*}
\sum_{i=0}^{t-1} x^\top_i \Sigma^{-1}_t x_i \leq d.
\end{align*}
\end{lemma}
\begin{proof}
Recall that $\Sigma_t = \sum_{i=0}^{t-1} x_i x_i^T + \lambda I_d$.
\begin{align*} \sum_{i=0}^{t-1} x^\top_i \Sigma^{-1}_t x_i   &= \sum_{i=0}^{t-1}\mathrm{tr}\left[ 
 \Sigma^{-1}_t x_ix_i^T \right]\\
 &= \mathrm{tr}\left[ 
 \Sigma^{-1}_t \sum_{i=0}^{t-1} x_ix_i^T \right] \\
 &= \mathrm{tr}\left[ 
 \Sigma^{-1}_t (\Sigma_t - \lambda I_d)\right]  \\
 &= \mathrm{tr}\left[I_d\right] - \mathrm{tr}\left[\lambda \Sigma^{-1}_t\right]\leq d.
 \end{align*}
 The last line follows from the fact that $\Sigma^{-1}_t$ is positive semidefinite.
\end{proof}

\begin{lemma}[Upper bound of $\sum_{i=0}^{t-1} x^\top_i \Sigma_i^{-1} x_i$ (adapted from \citet{abbasi2011improved})]\label{lem:sum_pos}
\begin{align*}
\sum_{i=0}^{t-1} x^\top_i \Sigma^{-1}_i x_i \leq 2d \log \left(1 + \frac{t C_b^2}{d \lambda} \right).
\end{align*}
\end{lemma}

\begin{proof}
First we prove that $\forall i \in \{0, 1,..., t-1\}, 0\leq x_i^\top \Sigma^{-1}_i x_i < 1$. Recall the definition of $\Sigma_i$ and we know $\Sigma^{-1}_i$ is a positive semidefinite matrix and thus $0 \leq x_i^\top \Sigma^{-1}_i x_i$. To prove $x_i^\top \Sigma^{-1}_i x_i < 1$, we need to decompose $\Sigma_i$ and write
\begin{align*}
\ x_i^\top \Sigma^{-1}_i x_i &= x_i^\top \left(\lambda I + \sum_{j=0}^{i-1} x_j x^\top_j \right)^{-1} x_i\\
&= x_i^\top \left(x_i x_i^\top - x_i x_i^\top + \lambda I + \sum_{j=0}^{i-1} x_j x^\top_j \right)^{-1} x_i.
\end{align*}
Let $A = - x_i x_i^\top + \lambda I + \sum_{j=0}^{i-1} x_j x^\top_j$ and it becomes
\begin{align*}
x^\top_i \Sigma^{-1}_i x_i = x^\top_i (x_i x^\top_i + A)^{-1} x_i.
\end{align*}
By Sherman-Morrison lemma (Lemma \ref{lem:sm}), we have
\begin{align*}
x^\top_i \Sigma^{-1}_i x_i &= x^\top_i \left(A^{-1} - \frac{A^{-1} x_i x^\top_i A^{-1}}{1 + x^\top_i A^{-1} x_i} \right) x_i\\
&= x^\top_i A^{-1} x_i - \frac{x^\top_i A^{-1} x_i x^\top_i A^{-1} x_i}{1 + x^\top_i A^{-1} x_i}\\
&= \frac{x^\top_i A^{-1} x_i}{1 + x^\top_i A^{-1} x_i} < 1.
\end{align*}
Next we use the fact that $\forall x \in [0, 1), x \leq 2\log(x+1)$ and we have
\begin{align*}
\sum_{i=0}^{t-1} x^\top_i \Sigma^{-1}_i x_i &\leq \sum_{i=0}^{t-1} 2\log \left(1+ x^\top_i \Sigma^{-1}_i x_i \right)\\
&\leq 2 \log \left( \frac{\det(\Sigma_{t-1})}{\det(\Sigma_0)} \right)\\
&\leq 2 d \log \left( 1 + \frac{t C^2_b}{d \lambda}\right),
\end{align*}
where the last two lines are by Lemma \ref{lem:det} and Lemma \ref{lem:potential}.
\end{proof}

\section{Technical Lemmas}\label{sec:tech_lem}

\begin{lemma}[Self-normalized bound for vector-valued martingales (Lemma A.9 of  \citet{agarwal2021rl})]\label{lem:self_norm}
Let $\{\eta_i\}_{i=1}^\infty$ be a real-valued stochastic process with corresponding filtration $\{\cF_i\}_{i=1}^\infty$ such that $\eta_i$ is $\cF_i$ measurable, $\E[\eta_i | \cF_{i-1}] = 0$, and $\eta_i$ is conditionally $\sigma$-sub-Gaussian with $\sigma \in \R^+$. Let $\{ X_i \}_{i=1}^\infty$ be a stochastic process with $X_i \in \cH$ (some Hilbert space) and $X_i$ being $\cF_t$ measurable. Assume that a linear operator $\Sigma: \cH \rightarrow \cH$ is positive deﬁnite, i.e., $x^\top \Sigma x > 0$ for any $x \in \cH$. For any $t$, define the linear operator $\Sigma_t = \Sigma_0 + \sum_{i=1}^t X_i X^\top_i$ (here $xx^\top$ denotes outer-product in $\cH$). With probability at least $1-\delta$, we have for all $t \geq 1$:
\begin{align*}
\left\| \sum_{i=1}^t X_i \eta_i \right\|^2_{\Sigma_t^{-1}} \leq \sigma^2 \log \left( \frac{\det(\Sigma_t) \det(\Sigma_0)^{-1}}{\delta^2}\right).  
\end{align*}
\end{lemma}

\begin{lemma}[Sherman-Morrison lemma \citep{sherman1950adjustment}]\label{lem:sm}
Let $A$ denote a matrix and $b,c$ denote two vectors. Then
\begin{align*}
(A + bc^\top)^{-1} = A^{-1} - \frac{A^{-1} bc^\top A^{-1}}{1+ c^\top A^{-1} b}.
\end{align*}
\end{lemma}

\begin{lemma}[Lemma 6.10 of \citet{agarwal2021rl}]\label{lem:det}
Define $u_t = \sqrt{x^\top_t \Sigma^{-1}_t x_t}$ and we have
\begin{align*}
\det \Sigma_T = \det \Sigma_0 \prod_{t=0}^{T-1} (1 + u^2_t). 
\end{align*}
\end{lemma}

\begin{lemma}[Potential function bound (Lemma 6.11 of \citet{agarwal2021rl})]\label{lem:potential}
For any sequence $x_0,...,x_{T-1}$ such that for $t< T, \|x_t\|_2 \leq C_b$, we have
\begin{align*}
\log \left( \frac{\det \Sigma_{T-1}}{\det \Sigma_0}\right) &= \log \det \left( I + \frac{1}{\lambda} \sum_{t=0}^{T-1} x_t x^\top_t \right)  \\
& \leq d\log\left(1+ \frac{TC_b^2}{d \lambda} \right). 
\end{align*}
\end{lemma}

\section{Conclusion}\label{sec:conlusion}

We study linear bandits with the underlying reward function being non-linear, which falls into the misspecified bandit framework. Existing work on misspecified bandit usually assumes uniform misspecification where the $\ell_\infty$ distance between the best-in-class function and the true function is upper bounded by the misspecification parameter $\epsilon$. Existing lower bound shows that the $\tilde{\Omega}(\epsilon T)$ term is unavoidable where $T$ is the time horizon, thus the regret bound is always linear. However, in solving optimization problems, one only cares about the approximation error near the global optimal point and approximation error is allowed to be large in highly suboptimal regions. In this paper, we capture this intuition and define a natural model of misspecification, called $\rho$-gap-adjusted misspecificaiton, which only requires the approximation error at each input $x$ to be proportional to the suboptimality gap at $x$ with $\rho$ being the proportion parameter. 

Previous work found that classical LinUCB algorithm is not robust in $\epsilon$-uniform misspecified linear bandit when $\epsilon$ is large. However, we show that LinUCB is automatically robust against such gap-adjusted misspecification. Under mild conditions, e.g., $\rho \leq O(1/\sqrt{\log T})$, we prove that it achieves the near-optimal $\tilde{O}(\sqrt{T})$ regret for problems that the best-known regret is almost linear. Also, LinUCB doesn't need the knowledge of $\rho$ to run. However, if the upper bound of $\rho$ is revealed to LinUCB, the $\beta_t$ term can be carefully chosen according to eq. \eqref{eq:known_rho}. Our technical novelty lies in a new self-bounding argument that bounds part of the regret due to misspecification by the regret itself, which can be of independent interest in more settings.

We believe our analysis for LinUCB is tight and the requirement that $\rho = O(1/\sqrt{\log T})$ is essential, but we conjecture that there is a different algorithm that could handle constant $\rho$ or even when $\rho$ approaches $1$ at a rate of $O(1/\sqrt{T})$. We leave the resolution to this conjecture as future work. For completeness, we include a simulation section in Appendix \ref{sec:simulation}.

More broadly, our paper opens a brand new door for research in model misspecification, including misspecified linear bandits, misspecified kernelized bandits, and even reinforcement learning with misspecified function approximation. Moreover, we hope our paper make people rethink about the relationship between function optimization and function approximation. In the future, much more can be done. For example, we can design a new no-regret algorithm that works under gap-adjusted misspecification framework where $\rho$ is a constant, and study $\rho$-gap-adjusted misspecified Gaussian process bandit optimization. 

\subsection*{Acknowledgments}
The work was partially supported by NSF Awards \#2007117 and \#2003257. We thank Ilija Bogunovic for the discussion at the early stage of this paper. Finally, we thank UAI reviewers and the area chair for their valuable input that led to improvements to the paper.

\bibliography{bib}

\begin{thebibliography}{24}
\providecommand{\natexlab}[1]{#1}
\providecommand{\url}[1]{\texttt{#1}}
\expandafter\ifx\csname urlstyle\endcsname\relax
  \providecommand{\doi}[1]{doi: #1}\else
  \providecommand{\doi}{doi: \begingroup \urlstyle{rm}\Url}\fi

\bibitem[Abbasi-yadkori et~al.(2011)Abbasi-yadkori, P\'{a}l, and
  Szepesv\'{a}ri]{abbasi2011improved}
Yasin Abbasi-yadkori, D\'{a}vid P\'{a}l, and Csaba Szepesv\'{a}ri.
\newblock Improved algorithms for linear stochastic bandits.
\newblock \emph{Advances in Neural Information Processing Systems}, 24, 2011.

\bibitem[Abe and Long(1999)]{abe1999associative}
Naoki Abe and Philip~M Long.
\newblock Associative reinforcement learning using linear probabilistic
  concepts.
\newblock In \emph{International Conference on Machine Learning}, 1999.

\bibitem[Agarwal et~al.(2014)Agarwal, Hsu, Kale, Langford, Li, and
  Schapire]{agarwal2014taming}
Alekh Agarwal, Daniel Hsu, Satyen Kale, John Langford, Lihong Li, and Robert
  Schapire.
\newblock Taming the monster: A fast and simple algorithm for contextual
  bandits.
\newblock In \emph{International Conference on Machine Learning}, 2014.

\bibitem[Agarwal et~al.(2021)Agarwal, Jiang, Kakade, and Sun]{agarwal2021rl}
Alekh Agarwal, Nan Jiang, Sham~M. Kakade, and Wen Sun.
\newblock Reinforcement learning: Theory and algorithms, 2021.

\bibitem[Alieva et~al.(2021)Alieva, Cutkosky, and Das]{alieva2021robust}
Ayya Alieva, Ashok Cutkosky, and Abhimanyu Das.
\newblock Robust pure exploration in linear bandits with limited budget.
\newblock In \emph{International Conference on Machine Learning}, 2021.

\bibitem[Auer et~al.(2002)Auer, Cesa-Bianchi, and Fischer]{auer2002finite}
Peter Auer, Nicolo Cesa-Bianchi, and Paul Fischer.
\newblock Finite-time analysis of the multiarmed bandit problem.
\newblock \emph{Machine learning}, 47:\penalty0 235--256, 2002.

\bibitem[Bogunovic and Krause(2021)]{bogunovic2021misspecified}
Ilija Bogunovic and Andreas Krause.
\newblock Misspecified gaussian process bandit optimization.
\newblock \emph{Advances in Neural Information Processing Systems}, 34, 2021.

\bibitem[Chu et~al.(2011)Chu, Li, Reyzin, and Schapire]{chu2011contextual}
Wei Chu, Lihong Li, Lev Reyzin, and Robert Schapire.
\newblock Contextual bandits with linear payoff functions.
\newblock In \emph{International Conference on Artificial Intelligence and
  Statistics}, 2011.

\bibitem[Claeys et~al.(2021)Claeys, Gancarski, Maumy-Bertrand, and
  Wassner]{claeys2021dynamic}
Emmanuelle Claeys, Pierre Gancarski, Myriam Maumy-Bertrand, and Hubert Wassner.
\newblock Dynamic allocation optimization in a/b-tests using
  classification-based preprocessing.
\newblock \emph{IEEE Transactions on Knowledge and Data Engineering},
  35\penalty0 (1):\penalty0 335--349, 2021.

\bibitem[Dani et~al.(2008)Dani, Hayes, and Kakade]{dani2008stochastic}
Varsha Dani, Thomas~P Hayes, and Sham~M Kakade.
\newblock Stochastic linear optimization under bandit feedback.
\newblock In \emph{Conference on Learning Theory}, 2008.

\bibitem[Foster and Rakhlin(2020)]{foster2020beyond}
Dylan Foster and Alexander Rakhlin.
\newblock Beyond ucb: Optimal and efficient contextual bandits with regression
  oracles.
\newblock In \emph{International Conference on Machine Learning}, 2020.

\bibitem[Foster et~al.(2018)Foster, Agarwal, Dudik, Luo, and
  Schapire]{foster2018practical}
Dylan Foster, Alekh Agarwal, Miroslav Dudik, Haipeng Luo, and Robert Schapire.
\newblock Practical contextual bandits with regression oracles.
\newblock In \emph{International Conference on Machine Learning}, 2018.

\bibitem[Foster et~al.(2020)Foster, Gentile, Mohri, and
  Zimmert]{foster2020adapting}
Dylan~J Foster, Claudio Gentile, Mehryar Mohri, and Julian Zimmert.
\newblock Adapting to misspecification in contextual bandits.
\newblock \emph{Advances in Neural Information Processing Systems}, 33, 2020.

\bibitem[Ghosh et~al.(2017)Ghosh, Chowdhury, and
  Gopalan]{ghosh2017misspecified}
Avishek Ghosh, Sayak~Ray Chowdhury, and Aditya Gopalan.
\newblock Misspecified linear bandits.
\newblock In \emph{AAAI Conference on Artificial Intelligence}, 2017.

\bibitem[Katz-Samuels et~al.(2020)Katz-Samuels, Jain, Jamieson,
  et~al.]{katz2020empirical}
Julian Katz-Samuels, Lalit Jain, Kevin~G Jamieson, et~al.
\newblock An empirical process approach to the union bound: Practical
  algorithms for combinatorial and linear bandits.
\newblock \emph{Advances in Neural Information Processing Systems}, 33, 2020.

\bibitem[Krishnamurthy et~al.(2021)Krishnamurthy, Hadad, and
  Athey]{krishnamurthy2021tractable}
Sanath~Kumar Krishnamurthy, Vitor Hadad, and Susan Athey.
\newblock Tractable contextual bandits beyond realizability.
\newblock In \emph{International Conference on Artificial Intelligence and
  Statistics}, 2021.

\bibitem[Lattimore and Szepesv{\'a}ri(2020)]{lattimore2020bandit}
Tor Lattimore and Csaba Szepesv{\'a}ri.
\newblock \emph{Bandit algorithms}.
\newblock Cambridge University Press, 2020.

\bibitem[Lattimore et~al.(2020)Lattimore, Szepesvari, and
  Weisz]{lattimore2020learning}
Tor Lattimore, Csaba Szepesvari, and Gellert Weisz.
\newblock Learning with good feature representations in bandits and in rl with
  a generative model.
\newblock In \emph{International Conference on Machine Learning}, 2020.

\bibitem[Li et~al.(2019)Li, Wang, and Zhou]{li2019nearly}
Yingkai Li, Yining Wang, and Yuan Zhou.
\newblock Nearly minimax-optimal regret for linearly parameterized bandits.
\newblock In \emph{Conference on Learning Theory}, 2019.

\bibitem[Moradipari et~al.(2020)Moradipari, Thrampoulidis, and
  Alizadeh]{moradipari2020stage}
Ahmadreza Moradipari, Christos Thrampoulidis, and Mahnoosh Alizadeh.
\newblock Stage-wise conservative linear bandits.
\newblock \emph{Advances in Neural Information Processing Systems}, 33, 2020.

\bibitem[Neu and Olkhovskaya(2020)]{neu2020efficient}
Gergely Neu and Julia Olkhovskaya.
\newblock Efficient and robust algorithms for adversarial linear contextual
  bandits.
\newblock In \emph{Conference on Learning Theory}, 2020.

\bibitem[Sherman and Morrison(1950)]{sherman1950adjustment}
Jack Sherman and Winifred~J Morrison.
\newblock Adjustment of an inverse matrix corresponding to a change in one
  element of a given matrix.
\newblock \emph{Annals of Mathematical Statistics}, 21\penalty0 (1):\penalty0
  124--127, 1950.

\bibitem[Wang et~al.(2021)Wang, Liu, Ge, Lian, and Zhang]{wang2021hybrid}
Shiyao Wang, Qi~Liu, Tiezheng Ge, Defu Lian, and Zhiqiang Zhang.
\newblock A hybrid bandit model with visual priors for creative ranking in
  display advertising.
\newblock In \emph{The Web Conference}, 2021.

\bibitem[Zanette et~al.(2020)Zanette, Lazaric, Kochenderfer, and
  Brunskill]{zanette2020learning}
Andrea Zanette, Alessandro Lazaric, Mykel Kochenderfer, and Emma Brunskill.
\newblock Learning near optimal policies with low inherent bellman error.
\newblock In \emph{International Conference on Machine Learning}, 2020.

\end{thebibliography}

\newpage 
\onecolumn
\appendix

\section{Proof of Proposition~\ref{prop:perservation}}\label{sec:pres}

Equivalently, $\rho$-gap-adjusted misspecification (Definition \ref{def:lm}) satisfies 
\begin{equation}\label{eqn:rho_miss}
 \left|f(x) - f_0(x) \right|\leq \rho    
\left|f^* - f_0(x)\right|,\;\;\forall x \in \cX.
\end{equation}

\begin{proof}[Proof of preservation of max value: $\max_{x\in\mathcal{X}}f(x)=f^*$]

Let $f^*_w := \max_{x\in\mathcal{X}}f(x)$. We first prove $f^*_w\leq f^*$ by contradiction. Suppose $f^*_w > f^*$, since $\mathcal{X}$ is compact, there exists $x_w\in\cX$ such that $f(x_w)=f^*_w>f^*$. Then by eq. \eqref{eqn:rho_miss} this implies
\[
f(x_w)-f_0(x_w)\leq \rho (f^*-f_0(x_w))\Rightarrow f^*<f^*_w=f(x_w)\leq \rho f^*+(1-\rho)f_0(x_w)\leq f^*
\]
Contraction! Therefore, $f_w^*\leq f^*$. On the other hand, choose $x_0\in\argmax_{x\in\cX}f_0(x)$, then by \eqref{eqn:rho_miss} $f(x_0)=f_0(x_0)=f^*$. This implies $f_w^*\geq f^*$. Combing both results to obtain $f_w^*= f^*$.
\end{proof}

\begin{proof}[Proof of preservation of maximizers: $\argmax_{x}f(x) =\argmax_{x}f_{0}(x)$]

Using that $f(x)\leq \rho f^*+(1-\rho)f_0(x)$ and $\max_{x\in\mathcal{X}}f(x)=f^*$, it is easy to verify $\argmax_{x}f(x) \subset\argmax_{x}f_{0}(x)$. On the other hand, if $x'\in\argmax_{x}f_{0}(x)$, then by eq. \eqref{eqn:rho_miss} $f(x')=f_0(x')=f^*$ and this means $\argmax_{x}f_0(x) \subset\argmax_{x}f(x)$. 
\end{proof}

\begin{proof}[Proof of self-bounding property]
This directly comes from the definition.
\end{proof}

\section{Property of Weak $\rho$-Gap-Adjusted Misspecification}\label{sec:weak}

First we recall Definition \ref{def:lm_weak}.

\begin{definition}[Restatement of Weak $\rho$-gap-adjusted misspecification]
Denote $f_w^*=\max_{x\in\mathcal{X}} f(x)$. Then we say $f$ is (weak) $\rho$-gap-adjusted misspecification approximation of $f_0$ for a parameter $0 \leq \rho < 1$ if:
\begin{align*}
\sup_{x \in \cX} \left| \frac{f(x) - f_w^*+f^*-f_0(x)}{f^* - f_0(x)}\right|\leq \rho.
\end{align*}
\end{definition}

Under the weak $\rho$-gap-adjusted misspecification condition, it no longer holds $f_w^*=f^*$. However, it still preserves the maximizers.

\begin{proposition}\label{prop_weka_rho}
 Under the weak $\rho$-gap-adjusted misspecification condition, it holds $$\argmax_{x}f(x) =\argmax_{x}f_{0}(x).$$
\end{proposition}
\begin{proof}
Suppose $x'\in\argmax_{x}f(x)$, then by definition
\[
|f^*-f_0(x')|=|f(x')-f_w^*+f^*-f_0(x')|\leq \rho |f^*-f_0(x')|\Rightarrow (1-\rho) |f^*-f_0(x')|\leq 0\Rightarrow x'\in\argmax_{x}f_0(x).
\]
On the other hand, if $x'\in\argmax_{x}f_0(x)$, then
\[
|f_w^*-f(x')|=|f(x')-f_w^*+f^*-f_0(x')|\leq \rho |f^*-f_0(x')|=0\Rightarrow x'\in\argmax_{x}f(x). 
\]
\end{proof}

The next proposition shows the weak $\rho$-adjusted misspecification condition characterizes the suboptimality gap between $f$ and $f_0$.

\begin{proposition}
    Denote $g(x):= f^*_w-f(x)\geq 0$, $g_0(x):=f^*-f_0(x)\geq 0$, then the weak $\rho$-gap-adjusted misspecification condition implies:
    \[
    (1-\rho)g_0(x)\leq g(x)\leq (1+\rho) g_0(x),\quad x\in\cX.
    \]
\end{proposition}
This can be proved directly by the triangular inequality. This reveals the weak $\rho$-gap-adjusted misspecification condition requires $g(x)$ to live in the band $[(1-\rho)g_0(x),(1+\rho) g_0(x)]$, and the concrete maximum values $f_w^*$ and $f^*$ can be arbitrarily different. 

\section{Linear Bandits under the Weak $\rho$-Gap-Adjusted Misspecification}\label{sec:weak_regret}

We need to slightly modify LinUCB \citep{abbasi2011improved} and work with the following LinUCBw algorithm.

\begin{algorithm}[!htbp]
\caption{LinUCBw (adapted from \citet{abbasi2011improved})}
	\label{alg:linucb2}
	{\bf Input:}
	Predefined sequence $\beta_t$ for $t=1,2,3,...$ as in eq. \eqref{eq:beta_t_2};
 Set $\lambda=\sigma^2/C^2_w$ and $\mathrm{Ball}_0 = \cW$.
	\begin{algorithmic}[1]
	    \FOR{$t = 0,1,2,... $}
	    \STATE Select $x_t=\argmax_{x \in \cX} \max_{[w^\top,c] \in \mathrm{Ball}_t} [w^\top,c] \begin{bmatrix}x\\1\end{bmatrix}$.
	    \STATE Observe $y_t = f_0(x_t) + \eta_t$.
     \STATE Update 
     \begin{align*}
\Sigma_{t+1} = \lambda I_{d+1} + \sum_{i=0}^{t} \begin{bmatrix}x_i\\1\end{bmatrix} \cdot [x^\top_i,1] \ \mathrm{where}\  \Sigma_0 = \lambda I_{d+1}.
\end{align*}
	    \STATE Update 
	    \begin{align*}
\begin{bmatrix}\hat{w}_{t+1}\\\hat{c}_{t+1}\end{bmatrix} = \argmin_{w,c} \lambda \left \|\begin{bmatrix}w\\c\end{bmatrix} \right\|^2_2+ \sum_{i=0}^{t} (w^\top x_i +c- y_i)^2_2.
\end{align*}
    \STATE Update
	    \begin{align*}
     \mathrm{Ball}_{t+1} = \left \{
    \begin{bmatrix}w\\c\end{bmatrix} \bigg\rvert \left\|\begin{bmatrix}w\\c\end{bmatrix} - \begin{bmatrix}\hat{w}_{t+1}\\\hat{c}_{t+1}\end{bmatrix} \right\|^2_{\Sigma_{t+1}} \leq \beta_{t+1} \right\}.
    \end{align*}
		\ENDFOR
	\end{algorithmic}
\end{algorithm}

\begin{theorem}\label{thm:2}
Suppose Assumptions \ref{ass:boundedness}, \ref{ass:unique}, and \ref{ass:rho} hold. W.l.o.g., assuming $c^*=f^*-f_w^*\leq F$. Set 
\begin{align}
\beta_t = 8\sigma^2 \left(1 + (d+1)\log\left(1+ \frac{t C^2_b (C^2_w+F^2) }{d \sigma^2} \right) + 2\log \left(\frac{\pi^2 t^2}{3\delta} \right)\right).\label{eq:beta_t_2}
\end{align} 
Then Algorithm~\ref{alg:linucb2} guarantees w.p. $> 1-\delta$ simultaneously for all $T=1,2,...$
\begin{align*}
R_T &\leq F +c^*+ \sqrt{\frac{8 (T-1) \beta_{T-1} (d+1)}{(1-\rho)^2} \log \left( 1 + \frac{T C^2_b (C^2_w+F^2) }{d \sigma^2 }\right)}.
\end{align*}
\end{theorem}

\begin{remark}
The result again shows that LinUCBw algorithm achieves $\tilde{O}(\sqrt{T})$ cumulative regret and thus it is also a no-regret algorithm under the weaker condition (Definition \ref{def:lm_weak}). Note Definition \ref{def:lm_weak} is quite weak which even doesn't require the true function sits within the approximation function class.
\end{remark}

\begin{proof}

The analysis is similar to the $\rho$-gap-adjusted case but includes $c^*=f^*-f^*_w$. For instance, let $\Delta^w_t$ denote the deviation term of our linear function from the true function at $x_t$, then
\begin{align*}
\Delta^w_t = f_0(x_t) - w^\top_* x_t-c^*,
\end{align*}
And our observation model (eq. \eqref{eq:obs}) becomes
\begin{align*}
y_t = f_0(x_t) + \eta_t = w_*^\top x_t + c^* + \Delta^w_t + \eta_t.
\end{align*}
Then similar to Lemma~\ref{lem:delta}, we have the following lemma, whose proof is nearly identical to Lemma~\ref{lem:delta}.
\begin{lemma}[Bound of deviation term]
$\forall t \in \{0,1,\ldots,T-1\}$,
\begin{align*}
|\Delta_t | \leq \frac{\rho}{1-\rho} w^\top_*(x_* - x_t).
\end{align*}
\end{lemma}

We also provide the following lemma, which is the counterpart of Lemma~\ref{lem:gap}.

\begin{lemma}
Define $u_t = \left \|\begin{bmatrix}x_t\\1\end{bmatrix} \right\|_{\Sigma_t^{-1}}$ and assume $\beta_t$ is chosen such that $w_*\in \mathrm{Ball}_t$.
Then
\begin{align*}
w_*^\top (x_* - x_t) \leq 2 \sqrt{\beta_t} u_t.
\end{align*}
\end{lemma}
\begin{proof}
Let $\tilde{w},\tilde{c}$ denote the parameter that achieves $\argmax_{w,c \in \mathrm{Ball}_t} w^\top x_t+c$, by the optimality of $x_t$, 
\begin{align*}
w_*^\top x_* - w^\top_* x_t &=\begin{bmatrix}w_*^\top,c^*\end{bmatrix} \begin{bmatrix}x_*\\1\end{bmatrix}-\begin{bmatrix}w_*^\top,c^*\end{bmatrix} \begin{bmatrix}x_t\\1\end{bmatrix}\\
&\leq \begin{bmatrix}\tilde{w}^\top,\tilde{c}\end{bmatrix} \begin{bmatrix}x_t\\1\end{bmatrix} - \begin{bmatrix}w_*^\top,c^*\end{bmatrix} \begin{bmatrix}x_t\\1\end{bmatrix}\\
&= (\begin{bmatrix}\tilde{w}^\top,\tilde{c}\end{bmatrix} - \begin{bmatrix}\hat{w}_t^\top,\hat{c}_t\end{bmatrix}+\begin{bmatrix}\hat{w}_t^\top,\hat{c}_t\end{bmatrix}-\begin{bmatrix}w_*^\top,c^*\end{bmatrix}) \begin{bmatrix}x_t\\1\end{bmatrix}\\
&\leq \left \|\begin{bmatrix}\tilde{w}^\top,\tilde{c}\end{bmatrix} - \begin{bmatrix}\hat{w}_t^\top,\hat{c}_t\end{bmatrix}\right\|_{\Sigma_t} \left \|\begin{bmatrix}x_t\\1\end{bmatrix} \right \|_{\Sigma^{-1}_t} + \left \|\begin{bmatrix}\hat{w}_t^\top,\hat{c}_t\end{bmatrix}-\begin{bmatrix}w_*^\top,c^*\end{bmatrix}\right \|_{\Sigma_t} \left \|\begin{bmatrix}x_t\\1\end{bmatrix}\right \|_{\Sigma^{-1}_t}\\
&\leq 2\sqrt{\beta_t} u_t
\end{align*}
where the second inequality applies Holder's inequality; the last line uses the definition of $\mathrm{Ball}_t$ (note that both $\begin{bmatrix}\tilde{w}^\top,\tilde{c}\end{bmatrix},\begin{bmatrix}w_*^\top,c^*\end{bmatrix}\in \mathrm{Ball}_t).$
\end{proof} 

The rest of the analysis follows the analysis of Theorem~\ref{thm:main}.
\end{proof}

\section{Simulation}\label{sec:simulation}

In this section, we run a simulation on a $1$-dimensional test function shown in Figure \ref{fig:exp1}. Here we run the first $10$ iterations with uniform sampling and the remaining $100$ iterations are using LinUCB algorithm. In Figure \ref{fig:exp2} we can see that cumulative regret is increasing with uniform sampling but it doesn't increase when running LinUCB. The reason behind it is that under the gap-adjusted misspecification, LinUCB is able to quickly find the optimal point $x_*=0$.

\begin{figure*}[!htbp]
	\centering 
	\subfigure[$1$-dimensional test function.]{\label{fig:exp1}\includegraphics[width=0.45\linewidth]{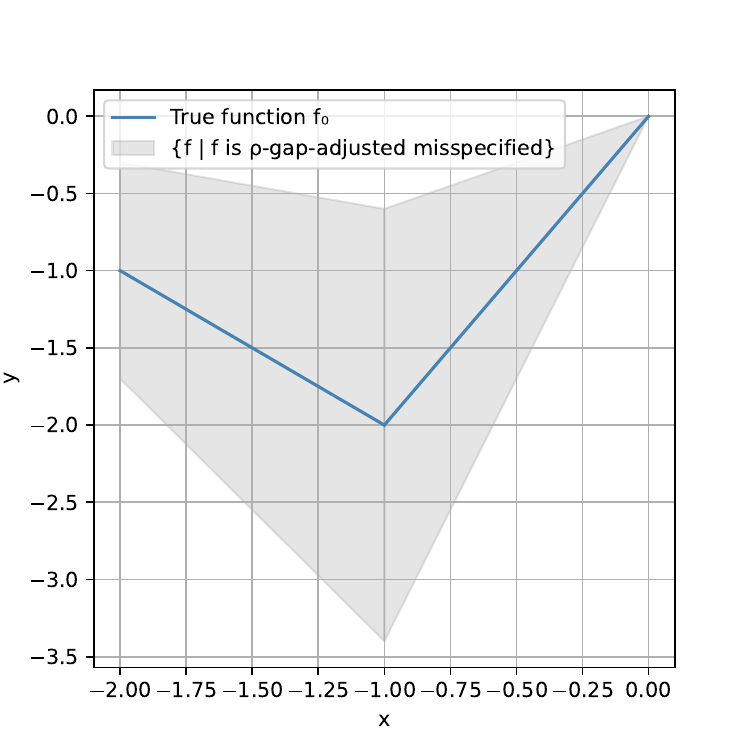}}
	\subfigure[Cumulative regret]{\label{fig:exp2}\includegraphics[width=0.45\linewidth]{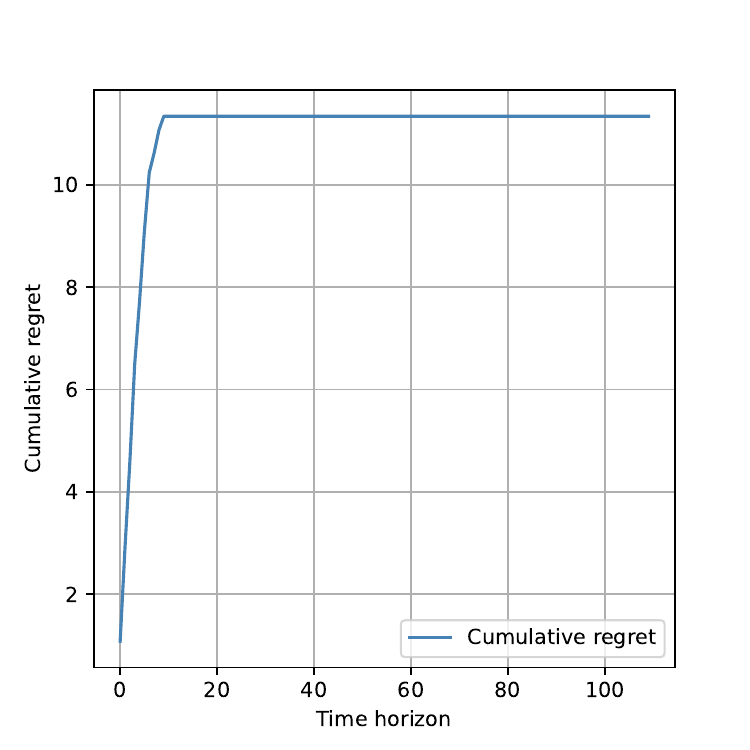}}
 \caption{Simulation function and result.}
\end{figure*}

\end{document}